\documentclass{article}
\usepackage{spconf,amsmath,amssymb,amsthm,booktabs,graphicx}
\usepackage{tikz}
\usepackage[hidelinks]{hyperref}
\newcommand{\TV}{d_{\mathrm{TV}}}
\newtheorem{theorem}{Theorem}

\begin{document}
\ninept
\title{ADAPTIVE SAFETY FILTERING FOR FROZEN ACC POLICIES VIA CONFORMAL RESIDUAL CALIBRATION}
\name{\begin{tabular}{@{}c@{}}
Zhiruo Zhou$^{1,2}$\raisebox{.35ex}{\scriptsize *}, Rigaudiere Z. Li$^{1}$\raisebox{.35ex}{\scriptsize *}\raisebox{.35ex}{\scriptsize †}, Chen Xiwen$^{1,3}$\raisebox{.35ex}{\scriptsize *},\\
Yucheng Chen$^{2}$, Xiaojun Zhu$^{1}$, Houde Liu$^{1}$\raisebox{.35ex}{\scriptsize †}
\end{tabular}
\thanks{\fontsize{9pt}{10.8pt}\selectfont \textsuperscript{*} These authors contributed equally to this work.\\
\textsuperscript{†} Corresponding authors: Rigaudiere Z. Li and Houde Liu.\\
This work was supported by the Shenzhen Science and Technology Program (Grant No. RCJC20210706091946001) and the Shenzhen Science and Technology Program (Grant No. ZDCY20250901104207008).}}
\address{$^{1}$Tsinghua Shenzhen International Graduate School, Tsinghua University, Shenzhen, China\\
$^{2}$Wuhan University of Technology, Wuhan, China\\
$^{3}$Shanghai Jiao Tong University, Shanghai, China}
\maketitle

\begin{abstract}
Frozen adaptive cruise control (ACC) policies can violate constraints when
deployment dynamics differ from their training conditions. We propose
residual-aware conformal action filtering (RACF), which calibrates residuals of
a fixed nominal predictor and converts their quantile into an operating margin
for finite-model action projection. Completed transitions update margins and
candidate selection without retraining the policy. In a registered comparison
over 2,400 controller--trial units, Adaptive RACF achieves 94.3\% episode
safety, improving by 19.9 percentage points over the evaluated nominal CBF-QP
baseline while reducing projection frequency from 8.11\% to 6.63\%. A
controlled study isolates a 4.54-point improvement from residual-margin
injection. In a separate matched-hardware evaluation, Adaptive reduces mean
amortized rollout time by 21.2\% relative to Robust CBF-QP, with 161/180 versus 170/180 safe
episodes. We characterize conditions linking one-step residual coverage to
constraint satisfaction and quantify the observed safety--computation
trade-offs.
\end{abstract}

\section{Introduction}
\label{sec:intro}
Offline reinforcement learning (RL) learns from previously collected data
\cite{levine2020offline}. IQL avoids evaluating unseen actions
\cite{kostrikov2022iql}, whereas CQL penalizes overestimated values
\cite{kumar2020cql}. Cross-domain methods address dynamics shifts during policy
learning \cite{qiao2026droco}. These approaches do not remove the need for a
runtime safety interface around a frozen controller. In ACC, mass, road grade,
and actuator gain alter ego dynamics, whereas braking and stop--go motion
change the exogenous lead trajectory, making grade--lead-motion combinations
hard cases. A deployed filter must absorb both forms of mismatch while limiting
policy distortion.

Runtime safety mechanisms include shielding \cite{alshiekh2018shielding} and
continuous-action safety layers \cite{dalal2018safe}. CBF-QPs
\cite{ames2017cbf} and predictive safety filters \cite{wabersich2021psf}
enforce model-based constraints.
Conformal prediction calibrates residual uncertainty
\cite{angelopoulos2023gentle}, including under distribution shift
\cite{tibshirani2019covariate,barber2023beyond}. Recent studies connect this
uncertainty to safe planning and control
\cite{lindemann2023safeplanning,zhou2024safetycritical}. Related work addresses
robust verification and shifted environments
\cite{zhao2024robuststl,beirami2026,rahaman2026shift}.
We study runtime action filtering for frozen ACC policies under dynamics shift.
Changes in vehicle parameters and lead-vehicle motion can make nominal
predictions inaccurate, motivating a filter that responds to observed mismatch
without retraining the controller. RACF calibrates residuals of a fixed nominal
predictor, maps the resulting quantile to an operating margin, and projects
policy actions through a finite set of dynamics constraints. Completed
transitions update the margin and candidate selection for subsequent actions.
This residual-to-action interface separates policy optimization from
deployment-time intervention and permits direct comparison of constraint
satisfaction, action modification, and computational cost.

\noindent\textbf{Contributions.} We make three contributions. First, we
introduce a residual-to-action interface that combines conformal residual
calibration, online margin updating, and candidate selection without changing
policy parameters. Second, we characterize sufficient conditions for
translating marginal one-step residual coverage into constraint satisfaction by
the final applied action. Third, paired ACC comparisons, controlled component
studies, and matched-hardware measurements distinguish margin benefits from
full-system safety and computational outcomes.

\begin{figure*}[t]
\centering
\includegraphics[width=\textwidth]{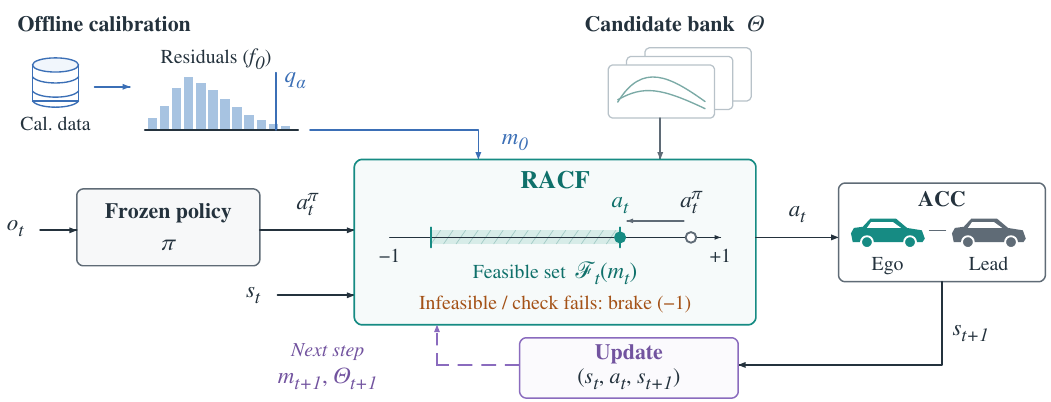}
\caption{RACF data flow. Offline calibration uses residuals of a fixed nominal
predictor to initialize the operating margin. At runtime, the policy proposes
$a_t^\pi$ from $o_t$, while the filter uses $s_t$ and selected candidate
constraints to compute $a_t$. Completed transitions update the margin and
candidate set for the next step. An infeasible solve or failed post-check
invokes bounded maximum braking.}
\label{fig:method}
\end{figure*}

\section{Problem Formulation}
Consider a finite-horizon controlled process with observation $o_t$ and frozen
policy $a_t^\pi=\pi(o_t)$. We seek an applied action $a_t$ in a bounded set
$\mathcal A$ that preserves the filter margin
$h_{\mathrm{filt}}(s_{t+1})\geq0$ while staying close to $a_t^\pi$. Policy
training precedes calibration and evaluation, and filtering does not update the
policy weights.

The ACC state is $s=(v_e,v_l,g,v_d,\phi)$, comprising ego speed, lead
speed, inter-vehicle gap, desired speed, and road grade. Calibration and
filtering use fixed dimensionless coordinates. Speeds are divided by
$u_v=1$~m/s, gap by $u_g=1$~m, and grade by $u_\phi=1^\circ$. These reference
units preserve the implemented numerical values; bars on normalized states are
omitted below.

The normalized action $a\in[-1,1]$ maps linearly to the evaluated actuator range
$[F_{\min},F_{\max}]=[-6000,4000]$ N via
$F(a)=F_{\min}+(a+1)(F_{\max}-F_{\min})/2=-1000+5000a$ N. Under an actuator-gain
shift $\gamma$, the simulator applies $F_\gamma(a)=-1000+5000\gamma a$ N.
At $\Delta t=0.1$ s, the predictor models drag, rolling resistance, road grade,
and the mass-dependent force-to-acceleration conversion
\cite{rajamani2012vehicle}. The score predictor $f_0$ uses a mass of 1500 kg,
the observed grade, actuator gain 0.7, and constant lead speed. Its drag and
rolling coefficients are 0.24 and 0.010, whereas the simulator uses 0.30 and
0.012. Candidate predictors $f_\theta$ vary mass, grade, and gain over a finite
set $\Theta$.

We define dimensionless CTH components
$\bar h_g=(g-d_0-T_hv_e)/u_g$ and
$\bar h_v=(v_{\max}-v_e)/u_v$ from the standard ACC spacing policy
$d_{\mathrm{CTH}}=d_0+T_hv_e$ \cite{wu2020spacing,molnar2023cruise}. We use
$d_0=3$ m, $T_h=1.5$ s, and $v_{\max}=30$ m/s, giving
\begin{equation}
 h_{\mathrm{filt}}(s)=\min\{\bar h_g(s),\bar h_v(s)\}.
 \label{eq:margin}
\end{equation}
Closing-speed and braking-distance models provide richer risk definitions
\cite{molnar2023cruise,shalevshwartz2017rss}. Under the
Euclidean norm in the fixed dimensionless coordinates, $h_{\mathrm{filt}}$ is
Lipschitz with
$L_h=\sqrt{1+(T_hu_v/u_g)^2}=\sqrt{3.25}$.
Changing reference units or coordinate weights requires recalibration of the
score, margin, and Lipschitz conversion.

The episode endpoint $E_{\mathrm{safe}}(\tau)$ equals one only when every
executed transition satisfies $\bar h_g(s_{t+1})\geq0$ and no collision occurs.
A collision-terminated episode is unsafe; no undefined post-termination state
is evaluated. The speed component $\bar h_v$ remains a filter constraint.
Speed tracking, headway-violation duration, and comfort are separate outcomes.

\section{RACF Method}
\subsection{Residual calibration and action projection}
RACF separates uncertainty calibration from policy optimization. Residuals of
the fixed nominal predictor summarize observed one-step mismatch, and their
calibrated quantile initializes the operating margin. Projection seeks the
smallest feasible change to the policy action under a finite bank of dynamics
predictors. Online residual windows let the margin respond to recent mismatch,
while candidate selection changes which hypotheses constrain the projection.
The predefined bank remains part of the interface, including residual tracking
and candidate recovery (Fig.~\ref{fig:method}). For an independently collected transition
$Z=(s,a,s^+)$ in the dimensionless coordinates of Sec.~2, define
\begin{equation}
 S_0(Z)=\|s^+-f_0(s,a)\|_2,\quad
 k=\lceil(n+1)(1-\alpha)\rceil,\quad q_\alpha=S_{0,(k)},
 \label{eq:quantile}
\end{equation}
where $S_{0,(k)}$ is the $k$th order statistic of the $n$ calibration scores;
$q_\alpha=+\infty$ when $k>n$. This score captures aggregate predictor mismatch,
including unmodeled lead motion. Conformal risk control provides broader
loss-based constructions \cite{angelopoulos2024crc}. In our experiments,
$n=500$ nominal calibration transitions and $\alpha=0.1$ yield
$q_\alpha=0.09188$. The independent one-step diagnostic uses uniformly random
cruise actions and is separate from the braking and stop--go policy rollouts.

RACF enforces the margin for every active predictor
$\{f_\theta:\theta\in\Theta_t\}$:
\begin{equation}
 \mathcal F_t(m_t)=\{a\in\mathcal A:
 h_{\mathrm{filt}}(f_\theta(s_t,a))\geq m_t,\ \theta\in\Theta_t\}.
 \label{eq:feasible}
\end{equation}
It then minimally modifies the frozen policy proposal:
\begin{equation}
 a_t=\arg\min_{a\in\mathcal F_t(m_t)}\|a-a_t^\pi\|_2^2.
 \label{eq:projection}
\end{equation}
The finite candidate set is
$\Theta=\{1500,1800\}\times\{0,3\}\times\{0.7,1.0\}$ over mass (kg),
grade (degrees), and actuator gain. Static RACF sets
$\Theta_t=\Theta$ and $m_t=\eta q_\alpha$. Base robust projection keeps the
candidate set, objective, bounds, solver, and fallback fixed while setting
$m_t=0$, thereby isolating the residual-margin effect. Default Static and
Adaptive RACF use $\eta=1$; the predictor-aligned comparison instead tests
$m=L_hq_\alpha$.

The score-defining predictor $f_0$ and the constraint predictors $f_\theta$ are
separate implementations. A matching parameter label or an implied constraint
does not establish functional equality to $f_0$.
Section~\ref{sec:coverage-control} states the alignment condition linking the
nominal score to safety. An empty feasible set or a failed post-solve check
invokes bounded maximum braking, which need not preserve this premise
(Fig.~\ref{fig:method}).

\subsection{Adaptive update and operating envelope}
Adaptive RACF updates the candidate set and margin from completed transitions.
After a 20-transition warm-up, it recomputes candidate-specific quantiles from
a 40-transition window every ten transitions and retains at most four
candidates within 0.02 normalized-residual units of the minimum. Observed grade
restricts the grid to values within $0.75^\circ$; if none qualify, that axis
remains unpruned. Outside a recovery period, an upward crossing of $2q_\alpha$
by the largest candidate residual restores the full grid and suspends pruning
for 20 completed transitions; per-step grade filtering continues.
Grade compatibility is applied before projection, and an empty intersection
restores all compatible candidates. Separately, a 100-transition window $W_t$
updates $m_t=\eta\max\{q_\alpha,\widehat q_\alpha(W_t)\}$ on the same schedule.
Each episode starts with empty windows, the full residual-tracker set, and
$m_0=\eta q_\alpha$. After transition $t$, the margin and selector updates run
in that order and affect action $t+1$. These empirical rules use completed
transitions only and do not inherit adaptive conformal validity
\cite{gibbs2021adaptive}.

For a formal finite-horizon construction, define
\begin{equation}
 M_H(\tau)=\max_{0\leq t<H}S_0(Z_t).
 \label{eq:maxscore}
\end{equation}
Envelope RACF uses the fixed margin $B=0.45$,
which is distinct from a calibrated $q_H$.

Nominal and robust CBF-QP minimize action deviation under discrete-time barrier
constraints \cite{agrawal2017discrete}:
$\bar h_j(f_\theta(s_t,a))\geq(1-\kappa)\bar h_j(s_t)$,
$j\in\{g,v\}$. The nominal version enforces one nominal hypothesis, whereas the
robust version enforces all eight. We use $\kappa=0.2$; all filters share the
same action bounds and bounded-braking fallback.

\section{Coverage and Conditional Safety}
\label{sec:coverage-control}
\subsection{One-step residual coverage}
Let $D=(S_{0,1},\ldots,S_{0,n})\sim P^n$ be i.i.d. calibration scores from $P$.
Conditional on $D$, let $Q_D$ be the law of the next deployed score $S_{0,*}$.
\begin{theorem}[Shifted marginal coverage]
For $q_\alpha(D)$ from \eqref{eq:quantile}, if
$\TV(P,Q_D)\leq\epsilon$ almost surely in $D$, then
\begin{equation}
 \mathrm{E}_D Q_D(({-}\infty,q_\alpha(D)])
 =\Pr_{D,S_{0,*}}\{S_{0,*}\leq q_\alpha(D)\}
 \geq1-\alpha-\epsilon .
 \label{eq:tv}
\end{equation}
\end{theorem}
\noindent\emph{Proof sketch.} An independent $\widetilde S_*\sim P$ gives the
marginal rank bound. Applying TV to $(-\infty,q_\alpha(D)]$ for each $D$ and
averaging subtracts at most $\epsilon$. TV alone does not supply the rank bound.
This is marginal over $D$ and $S_{0,*}$; it is neither state/history-conditional
nor a simultaneous multi-step statement.

\subsection{Residual-to-safety conversion}
If the final applied action, including fallback, satisfies
$h_{\mathrm{filt}}(f_0(s_t,a_t))\geq m_t\geq L_hq_\alpha$, then on
$S_{0,*}\leq q_\alpha$, Lipschitz continuity gives
$h_{\mathrm{filt}}(s_{t+1})\geq0$. Under Theorem~1 this yields the same
$1-\alpha-\epsilon$ one-step lower bound. Other candidate constraints require
their own discrepancy bound.
On the ACC state domain $v_e\geq0$, $h_{\mathrm{filt}}\geq0$ implies
$g\geq d_0=3$ m, above the 0.5 m collision threshold; outside this domain the
result is only a barrier-satisfaction bound.

\begin{center}
\centering
\setlength{\tabcolsep}{4pt}
\renewcommand{\arraystretch}{0.92}
\begin{tabular}{lll}
\toprule
Condition & Role & Current status \\
\midrule
IID + TV & Assumed & Not estimated \\
$L_h$ & Analytic & Coordinate-derived \\
$f_0$, $m=L_hq_\alpha$ & Structural & Trace-checked \\
Feas. + fallback & Assumed & Violations logged \\
\bottomrule
\end{tabular}
\end{center}

\subsection{Practical realizability and premises}
Episode-level calibration requires a frozen filter, horizon, reset law, and
termination rule. For independent episode maxima $D_H\sim P_H^{n_H}$, a
finite-sample $q_H$ also requires a shift bound to the deployment episode law.
Changing the filter using $q_H$ changes that law and requires fresh calibration
or another bound; Fig.~\ref{fig:evidence}(c) contrasts one-step and
policy-conditioned maxima.

Theorem~1 concerns marginal one-step residual coverage under explicit sampling
and distribution-shift assumptions. Constraint satisfaction additionally
requires the final applied action to satisfy the aligned, Lipschitz-scaled
nominal constraint, including on fallback steps. Empirical online updates do
not automatically preserve these premises. We therefore separate statistical
coverage from executable action conditions; rollout and trace analyses assess
observed behavior rather than establish a closed-loop certificate.

\begin{figure*}[t]
\centering
\includegraphics[width=\textwidth]{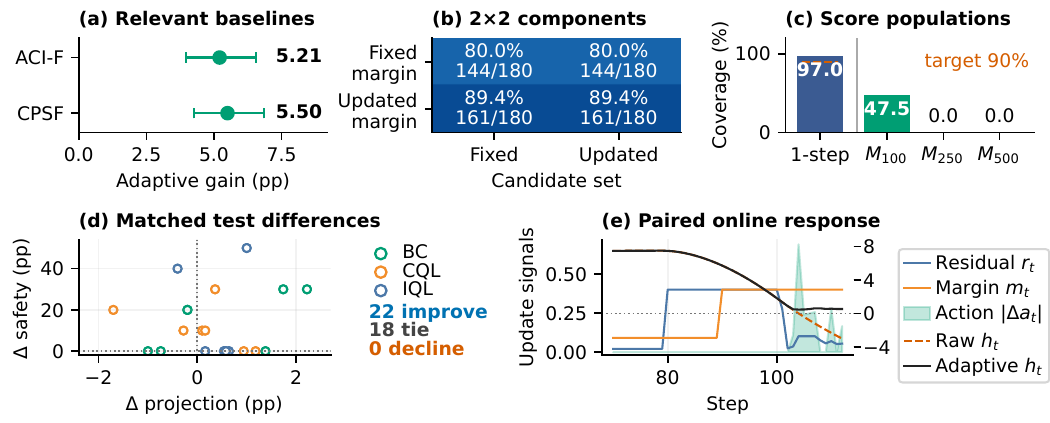}
\caption{ACC \emph{partial-rollout} evidence. (a) Adaptive-minus-control safety
(pp), 95\% shared-reset intervals; positive favors Adaptive. (b) Matched
candidate-set$\times$margin ablation ($N=180$), safety and Safe/$N$.
(c) Independent one-step versus policy-induced episode-max $M_H$ coverage;
the 90\% target is one-step only. (d) Adaptive-minus-Static safety/projection,
40 development-selected choices; colors denote policy families. (e) Paired
trace: $r_t,m_t,|\Delta a_t|$ (left axis), $h_t$ (right axis).}
\label{fig:evidence}
\end{figure*}

\section{Experiments}
\label{sec:experiments}

\subsection{Evaluation setup}
All experiments are ACC \emph{partial rollouts}. Paired comparisons share
policy, condition, scenario, trial, seed, initial state,
and lead trajectory. Policy input is its training observation; the filter uses
simulator state. Episode safety is $E_{\mathrm{safe}}(\tau)=1$.

The registered comparison uses four IQL, four CQL, three BC, and one IDM
controller over four physical conditions, two lead scenarios, and 25 shared
resets per cell (2,400 units; $H=500$). Collision at $g\leq0.5$ m ends an
episode. We report uncertainty intervals \cite{agarwal2021precipice}.
Two-sided 95\% CIs use 10,000 bootstrap draws of shared reset
clusters indexed by condition, scenario, trial, and seed, each carrying all
policies.

With $(m,\phi,\gamma)$ in kg, degrees, and actuator gain, nominal and mass
shifts use $(1500,0,1.0)$ and $(1800,0,1.0)$. Slope and combined shifts use
$(1500,3,1.0)$ and $(1800,3,0.7)$. Lead scenarios are brake and stop--go;
state-estimation error is outside this study.

\subsection{Registered comparison}
Table~\ref{tab:main-results} compares complete configurations. Adaptive is safe
in 2262/2400 episodes (94.3\%) versus 1785/2400 (74.4\%) for nominal CBF-QP, a
19.9-pp gain. Projection frequency decreases from 8.11\% to 6.63\%.

All five methods have zero observed collisions. Fig.~\ref{fig:evidence}(a)
reports separate Adaptive gains over ACI
and CPSF-style of 5.21 pp [3.96, 6.54] and 5.50 pp [4.25, 6.83].

ACI \cite{gibbs2021adaptive} keeps $\Theta$ and updates
$\alpha_{t+1}=\Pi_{[.01,.50]}[\alpha_t+.01(.10-I_t)]$ from a fixed 500-score
pool using $I_t=1\{S_0>q_{\alpha_t}\}$. CPSF-style
\cite{wabersich2021psf} keeps candidates within $q_0=0.09188$ of $f_0$ at
$a_t^\pi$, restoring $\Theta$ if empty. Both share bounds, state, solver, and
fallback; they are frozen before test and are one-step interfaces, not
multi-step PSF reproductions. Intervention rates are outcomes.

\begin{table}[t]
\caption{Separate cohorts. Registered: $N=2400$; Proj./Corr./Jerk are mean
episode projection (\%), normalized correction, and $P_{95}$ jerk (m/s$^3$).
Matched hardware: $N=180$; time is episode wall time divided by executed steps,
averaged over episodes (Supplement Sec.~3.2).}
\label{tab:main-results}
\centering
\setlength{\tabcolsep}{1.7pt}
\renewcommand{\arraystretch}{0.88}
\begin{tabular}{lrrrr}
\toprule
Method & Safe/$N$ & Proj. & Corr. & Jerk \\
\midrule
Nominal CBF-QP & 1785/2400 & 8.11 & 0.0177 & 2.64 \\
Static RACF & 2133/2400 & 6.51 & 0.0172 & 3.13 \\
Adaptive RACF & 2262/2400 & 6.63 & 0.0183 & 3.24 \\
Envelope RACF & 2287/2400 & 6.77 & 0.0188 & 3.36 \\
Robust CBF-QP & 2375/2400 & 8.02 & 0.0179 & 3.03 \\
\addlinespace
\multicolumn{5}{l}{\emph{Matched-hardware cohort ($N=180$)}} \\
Method & Safe/$N$ & \multicolumn{3}{c}{Runtime (ms/step)} \\
Adaptive RACF & 161/180 & \multicolumn{3}{c}{2.434} \\
Envelope RACF & 164/180 & \multicolumn{3}{c}{4.082} \\
Robust CBF-QP & 170/180 & \multicolumn{3}{c}{3.090} \\
\bottomrule
\end{tabular}
\end{table}

\subsection{Controlled component studies}
With candidates, objective, solver, bounds, and fallback fixed, residual-margin
injection raises safety from 1993/2400 to 2102/2400: +4.54 pp [3.63, 5.54].
In a separate 180-unit study, candidate-only and Static reach 144/180, whereas
margin-only and joint reach 161/180 (Fig.~\ref{fig:evidence}(b)); candidate
selection has no independent binary safety gain in this cohort.

Another development/test split gives 162/180, 156/180, 157/180, and 158/180
for selected fixed, rolling, ACI, and Adaptive; Adaptive minus fixed is
$-2.22$ pp [$-5.56$, 0.56], showing neither superiority nor equivalence. An
offline floor raises rolling/ACI coverage from 83.6/84.3\% to 90.4/90.8\%
without changing paired safety (Supplement Sec.~3.2).

\subsection{Computational cost and diagnostics}
The balanced-order evaluation in Table~\ref{tab:main-results} measures mean
amortized rollout time. Adaptive's mean is 21.2\% lower than Robust's, with
5.00 pp lower safety. This comparison of complete implementations does not
isolate candidate selection.

Supplement Secs.~2.3, 2.5, 2.7, and 2.10 provide matched-point, coverage,
out-of-grid, and typed-fallback diagnostics.

\section{Discussion and Conclusion}
\label{sec:conclusion}
RACF supplies residual-driven action filtering for frozen ACC policies. It has
higher episode safety than the evaluated nominal CBF-QP baseline, and controlled
tests identify a residual-margin benefit. Adaptive has lower mean amortized
rollout time, while Robust CBF-QP has higher episode safety in that cohort.
Selected fixed margins remain competitive, so the
evidence does not establish a universal advantage for online adaptation.

The conditional analysis links one-step coverage to action execution under
explicit premises. Evidence remains limited to simulated ACC, an eight-model
library, and simulator-state access; larger libraries, estimated states, and
target-hardware evaluation remain future work.

\newpage
\bibliographystyle{IEEEbib}
\bibliography{references}

\end{document}

\ninept
\maketitle
\section{Supplementary\\Experiments}
\suppressfloats[t]
\subsection{Evaluation protocol}
\label{sec:supp-protocol}
All studies below are ACC \emph{partial rollouts}. Compared methods within a
cohort share initial physical states, environment seeds, policy checkpoints,
inference settings, and exogenous lead trajectories. Policies receive their
training-time observations, whereas filters receive simulator state. Outcomes
include collision, violation, fallback, early termination, solver status,
tracking, jerk, intervention, and latency. Exact McNemar tests are supplementary:
they do not account for cross-policy dependence from shared resets. Binary and
continuous effect intervals use
10,000 condition--scenario-stratified shared-reset-cluster bootstrap draws that
carry all fixed policies jointly, with Holm correction within each comparison family.

The main comparison contains 12 frozen controllers, four physical conditions, two
lead scenarios, 25 resets per cell, and $H=500$. The residual-margin study
uses 2,400 fresh paired units; the adaptive study uses 720 units and 80
environment seeds disjoint from calibration and development; the horizon cohort
uses 120 units with shared trajectory prefixes. Separate development ablations
use nine policies, two conditions, two scenarios, ten resets, and $H=500$.
The action range is $[-1,1]$, and all filters use the same bounded-braking
fallback. Static/Adaptive RACF use $q_\alpha=0.0918766$, Envelope RACF uses
fixed $B=0.45$, and CBF-QP uses $\kappa=0.2$.
Figure~\ref{fig:supp-diagnostics} presents policy-group, physical-condition,
one-step-coverage, and action-noise slices.
The controller artifact types and seed provenance are listed in the companion
reproducibility record (with SHA-256 asset manifest); all analysis below uses
the same frozen assets.

\begin{figure*}[t]
\centering
\includegraphics[width=\textwidth]{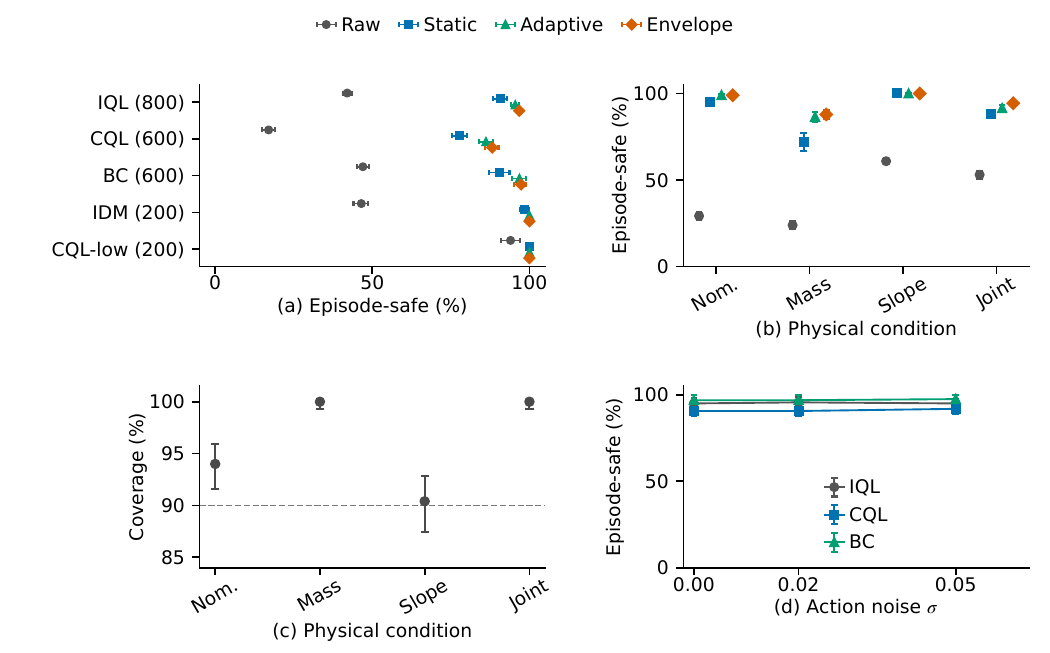}
\caption{Additional ACC results. Panels (a--b) show episode-safe rates
with 95\% shared-reset cluster-bootstrap intervals for policy groups and
physical conditions in the main
12-controller primary matrix. IQL/BC use seeds 0--3/0--2 ($N=800/600$);
CQL uses mixed-data seeds 0--2 ($N=600$), CQL-low the suboptimal-data seed-0
checkpoint ($N=200$), and IDM has $N=200$. CQL plus CQL-low forms the CQL
family in Table~\ref{tab:policy-family}.
(c) Independent one-step nominal-residual coverage on 500 disjoint transitions
per condition, with exact 95\% binomial intervals, target $1-\alpha=0.9$, and
an 84--101.5\% detail axis.
(d) Adaptive-RACF episode-safe rate with 95\% shared-reset cluster-bootstrap
intervals for IQL, CQL, and BC under action noise in the six-policy cohort.
Color denotes methods in (a--b), while the inset legend denotes policy groups
in (d).}
\label{fig:supp-diagnostics}
\end{figure*}

\subsection{Controlled RACF decomposition}
The controlled comparison keeps the eight-candidate projection, dynamics,
solver, bounds, and fallback fixed. The zero-margin projection obtains
1,993/2,400 episode-safe rollouts; adding $m=q_\alpha$ produces 2,102/2,400,
a paired gain of 4.54 points (95\% CI [3.63,5.54]). There are 110 improving
pairs and one reversing pair, exact McNemar $p=8.63\times10^{-32}$, and zero
collisions for both methods.

\begin{table}[!h]
\caption{Controlled residual-margin comparison on 2,400 controller--trials.
Safe is the exact safe-episode count; Proj. and Viol. are observed-step
fractions; Jerk is the mean episode $P_{95}$ in m/s$^3$.}
\label{tab:pre-racf}
\centering
\begin{tabular}{lrrrr}
\toprule
Method & Safe & Proj. & Viol. & Jerk \\
\midrule
Base robust projection & 1993 & .0669 & .00323 & 2.858 \\
$+\,$Static RACF & 2102 & .0635 & .00260 & 3.098 \\
\bottomrule
\end{tabular}
\end{table}

The margin lowers projection frequency by 0.338 percentage points and violation
burden by 0.063 points. Jerk P95 rises by 0.240~m/s$^3$, speed RMSE by
0.002~m/s, and gap RMSE by 0.050~m, exposing the associated control cost.

The independent matched cohort evaluates the complete adaptive
candidate-and-margin update. Mean correction differs by 0.71\% and projection
frequency by 0.254 percentage points from static margin 0.25, satisfying the matching
criterion. Adaptive obtains 622/720 episode-safe rollouts versus 595/720, a
3.75-point gain (95\% CI [2.36,5.28]), with 30 improving and three reversing
pairs, $p=1.40\times10^{-6}$, zero collisions, and speed/gap RMSE ratios of
1.0005/1.0044. Eight policy strata meet the projection-frequency tolerance;
BC seed~1 lies outside it, so the comparison is interpreted at the aggregate level.

\subsection{Absolute filters and candidate set}
\begin{table}[!h]
\caption{Main comparison on 2,400 controller--trial units. The lower block
selects Envelope--Robust costs; Table~\ref{tab:full-registered} includes Adaptive.
Corr. is normalized action-correction
magnitude, Fall. is an observed-step fraction, Jerk is episode $P_{95}$ in
m/s$^3$, Speed is speed RMSE in m/s, and Gap is gap RMSE in m.}
\label{tab:external-baselines}
\centering
\begin{tabular}{lrr}
\toprule
Method & Episode-safe & Collisions \\
\midrule
Raw & 1001 & 904 \\
Nominal QP & 1437 & 0 \\
Nominal CBF-QP & 1785 & 0 \\
Static RACF & 2133 & 0 \\
Adaptive RACF & 2262 & 0 \\
Envelope RACF & 2287 & 0 \\
Robust CBF-QP & 2375 & 0 \\
\bottomrule
\end{tabular}
\par\medskip
\setlength{\tabcolsep}{2.5pt}
\begin{tabular}{lrrrrr}
\toprule
\multicolumn{6}{c}{Descriptive control-cost means} \\
\midrule
Method & Corr. & Fall. & Jerk & Speed & Gap \\
\midrule
Envelope RACF & .0188 & .0097 & 3.356 & 6.706 & 29.079 \\
Robust CBF-QP & .0179 & .0073 & 3.030 & 6.696 & 28.894 \\
\bottomrule
\end{tabular}
\end{table}

Robust CBF-QP enforces all eight candidate barrier constraints and has the
highest episode-safe rate among evaluated methods in the registered cohort.
The controlled decomposition isolates adding the calibrated
residual margin to otherwise identical robust projection; it does not compare
conformal calibration with alternative margin-selection procedures. In the candidate deletion,
only removing the mass axis changes the endpoint; removing grade or gain leaves
it unchanged, while the no-mass and nominal sets share the lower endpoint.

\begin{table}[!h]
\caption{Candidate-set deletion at fixed $B=0.45$. Viol. steps is a count;
Proj. and Fallback are observed-step fractions; Safe/120 is the exact
safe-episode count.}
\label{tab:candidate-deletion}
\centering
\begin{tabular}{lrrrr}
\toprule
Set & Safe/120 & Viol. steps & Proj. & Fallback \\
\midrule
Full grid & 109 & 119 & .0701 & .0135 \\
No mass & 107 & 130 & .0776 & .0137 \\
No grade & 109 & 119 & .0701 & .0135 \\
No gain & 109 & 119 & .0713 & .0127 \\
Nominal & 107 & 130 & .0714 & .0133 \\
\bottomrule
\end{tabular}
\end{table}

\subsection{Safety-first control-cost selection}
The smoothing development study contains 4,320 method-episodes, 360 paired
policy--trial units, and 40 shared resets. Adding the residual margin raises
episode-safe count from 249/360 to 277/360 (+7.78 points
[5.00,10.56]); fixed $B=0.45$ reaches 320/360 (+11.94 points
[8.61,15.28]). \mbox{A matched comparison} gives Adaptive 318/360 versus
static 302/360 (+4.44 points [2.50,6.67]), accompanied by a 9.11\% jerk
increase. The separate adaptive study above shows the same direction on
independent resets.

We then add $\lambda\|a-a_{t-1}\|_2^2$ while retaining the fixed envelope and
hard predicted-safety constraints. The selection criterion requires safety within one
point, zero collisions, at least 10\% jerk reduction, and speed/gap RMSE within
5\% of the envelope.
The smallest weight reduces jerk by 7.81\% while retaining safety within
1.11 points; larger weights yield 31.4--69.6\% jerk reductions with broader
safety and tracking tradeoffs. None of the three smoothing weights meets all
four requirements, so subsequent experiments use the original envelope and
preserve the predicted-safety priority.
\begin{table}[!h]
\caption{Smoothing screen relative to $\lambda=0$. Safe count is out of 360,
$\Delta$Safe is in percentage points, and Jerk/Speed/Gap are ratios. All
variants complete 360 rollouts with zero collisions.}
\label{tab:a100-slew-screening}
\centering
\begin{tabular}{lrrrrr}
\toprule
$\lambda$ & Safe count & $\Delta$Safe (pp) & Jerk & Speed & Gap \\
\midrule
0 & 320 & 0.00 & 1.000 & 1.000 & 1.000 \\
0.1 & 316 & $-1.11$ & .922 & 1.001 & 1.009 \\
1 & 293 & $-7.50$ & .686 & 1.011 & 1.084 \\
10 & 296 & $-6.67$ & .304 & 1.155 & 2.192 \\
\bottomrule
\end{tabular}
\end{table}

\subsection{Coverage, horizon, and deployment results}
Eight unique $(\alpha,n,H)$ configurations and the $B=0.45$ reference give
1,080 method-episodes. High marginal coverage and trajectory coverage are
reported separately:

Across 27 policy--configuration batches (1,080 episode rows and 120 paired
reset units), 184 configuration-level observations have at least 90\%
independent one-step coverage but an episode-level violation. Trace
analysis places 97.48\% of exceedances in lead-acceleration intervals.
The constant-lead predictor produces 0.401--0.413 episode maxima under
$-4$~m/s$^2$ braking versus $q=0.092$. Retrospectively adding lead modes
$\{-4,0,2\}$ reduces the maximum to 0.0035 and covers 120/120 inspected
episodes, motivating dedicated disturbance-aware recalibration on disjoint
complete trajectories.

From $H=100$ to 500, episode-safe rates change from 66.7 to 7.5\% for Raw,
90 to 65\% for Base projection, 97.5 to 80\% for Static, and 100 to 87.5\%
for Adaptive/Envelope; Robust CBF-QP remains at 98.3\%. The $H=250$ values
equal the $H=500$ values in this cohort. Under action noise
$\sigma=0,0.02,0.05$, Adaptive RACF retains positive gains over Raw across
IQL/CQL/BC, while disagreement rises to approximately 0, 0.016, and 0.040.
Static/Adaptive/Envelope warmed P95 latencies are 5.72/3.10/5.67~ms on RTX
4090.

A 20-step OSQP MPC solves all 100,000 programs but obtains 0/200
headway-safe trajectories, so it remains a model-mismatch comparison outside the
ranking.

\begin{table}[!h]
\caption{Sensitivity results. Setting encodes $\alpha/n/H$ for conformal
rows and $B,H$ for the reference. The quantity $q$ is a dimensionless residual
threshold; 1-step, MaxCov, and Safe are fractions, and Viol. counts violation
steps.}
\label{tab:sensitivity}
\centering
\setlength{\tabcolsep}{2.5pt}
\begin{tabular}{lrrrrr}
\toprule
Setting & $q$ & 1-step & MaxCov & Safe & Viol. \\
\midrule
.05/500/500 & .0962 & .989 & .000 & .800 & 211 \\
.10/500/500 & .0920 & .970 & .000 & .792 & 219 \\
.20/500/500 & .0833 & .939 & .000 & .783 & 225 \\
.10/100/500 & .0882 & .959 & .000 & .800 & 214 \\
.10/250/500 & .0898 & .963 & .000 & .800 & 222 \\
.10/1000/500 & .0908 & .964 & .000 & .800 & 221 \\
.10/500/100 & .0920 & .970 & .475 & .958 & 18 \\
.10/500/250 & .0920 & .970 & .000 & .792 & 219 \\
$B=.45,H=500$ & .4500 & 1.000 & 1.000 & .925 & 81 \\
\bottomrule
\end{tabular}
\end{table}

\par\medskip
\noindent\begin{minipage}{\columnwidth}
\refstepcounter{figure}\label{fig:controlled-effects}
\centering
\includegraphics[width=\columnwidth]{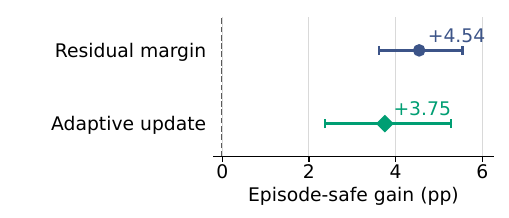}
\par\smallskip
\raggedright\textbf{Fig. \thefigure.} Controlled episode-safe effects in two distinct ACC
\emph{partial-rollout} studies: residual-margin injection ($N=2400$) and joint
candidate--margin updating versus matched Static ($N=720$). Whiskers are 95\%
shared-reset cluster-bootstrap intervals; the latter is not margin-only.
\end{minipage}
\par\medskip

\refstepcounter{table}
\noindent\textbf{Table \thetable.} Robust CBF-QP minus Envelope RACF on 2,400
reset/trial-matched units from separate deterministic runs. Differences are
descriptive, not stepwise paired; positive safety values favor CBF-QP.
\par\smallskip
\begingroup\ninept
\centering
\setlength{\tabcolsep}{4pt}
\begin{tabular}{lr}
\toprule
Metric & Robust CBF--Envelope \\
\midrule
Episode-safe (pp) & 3.67 \\
Correction & -0.0009 \\
Projection (pp) & 1.25 \\
Fallback (pp) & -0.24 \\
Jerk P95 & -0.327 \\
Gap RMSE & -0.185 \\
Speed RMSE & -0.010 \\
Latency P95 (ms) & -0.029 \\
\bottomrule
\end{tabular}

\endgroup
\par\medskip

\refstepcounter{table}
\noindent\textbf{Table \thetable.} Episode-safe rates (\%) by physical condition
and lead scenario for Raw and Envelope; 300 controller--trial units per row.
$\Delta$ is Envelope minus Raw (pp), computed before rounding.
\par\smallskip
\begingroup\ninept
\centering
\setlength{\tabcolsep}{3pt}
\begin{tabular}{lrrr}
\toprule
Condition/scenario & Raw & Envelope & $\Delta$ \\
\midrule
Nominal/brake & 50.0 & 100.0 & +50.0 \\
Nominal/stop--go & 8.3 & 98.0 & +89.7 \\
Mass/brake & 42.3 & 93.3 & +51.0 \\
Mass/stop--go & 5.3 & 82.3 & +77.0 \\
Slope/brake & 69.7 & 100.0 & +30.3 \\
Slope/stop--go & 52.0 & 100.0 & +48.0 \\
Joint/brake & 69.7 & 97.0 & +27.3 \\
Joint/stop--go & 36.3 & 91.7 & +55.3 \\
\bottomrule
\end{tabular}

\endgroup
\par\medskip

\newpage
\refstepcounter{table}
\noindent\textbf{Table \thetable.} Initial-state stratification of the
Static--Base residual-margin effect. Cuts use pooled medians of 6.618~m
headway and $-0.0719$~m/s closing speed; rows overlap and are not additive; entries are descriptive paired safety
differences. Marginal and intersection rows overlap and are not additive.
\par\smallskip
\begingroup\ninept
\centering
\setlength{\tabcolsep}{3pt}
\begin{tabular}{lrr}
\toprule
Initial-state stratum & Pairs & Static--Base (pp) \\
\midrule
Small headway, high closing & 696 & 6.61 \\
High closing speed & 1200 & 6.67 \\
Small headway & 1200 & 4.92 \\
Large headway & 1200 & 4.17 \\
Low closing speed & 1200 & 2.42 \\
Large headway, low closing & 696 & 2.30 \\
\bottomrule
\end{tabular}

\endgroup
\par\medskip

\refstepcounter{table}
\noindent\textbf{Table \thetable.} Paired episode-outcome changes on shared
policy, condition, scenario, reset and seed: Raw-unsafe to filter-safe, and the
reverse. These count repaired episodes, not within-trajectory recovery events.
\par\smallskip
\begingroup\ninept
\centering
\setlength{\tabcolsep}{4pt}
\begin{tabular}{lrr}
\toprule
Filter & Repaired & Reverse \\
\midrule
Static RACF & 1132 & 0 \\
Adaptive RACF & 1261 & 0 \\
Envelope RACF & 1286 & 0 \\
\bottomrule
\end{tabular}

\endgroup
\par\medskip

\subsection{Implementation details}
Brake episodes start from $(v_e,v_l,g)=(15,15,32)$ in m/s, m/s, and m; the
lead vehicle applies $-4$~m/s$^2$ over 10--12~s and $+2$~m/s$^2$ over
22--25~s.

Stop--go episodes start from $(10,10,25)$ and apply the same
accelerations over 8--10.5~s and 18--23~s. Randomized resets add independent
uniform offsets in $[-1.5,1.5]$~m/s to both speeds and multiply gap by a
uniform factor in $[0.85,1.15]$. Collision terminates an episode at
$g\leq0.5$~m; otherwise evaluation ends at the selected horizon.

The policy observes normalized ego-speed error, closing speed, gap, grade, and
previous action. The filter uses simulator state; this evaluation excludes
sensor and state-estimation error. Adaptive RACF computes
candidate-specific $\alpha=0.1$ quantiles from a 40-transition window every ten
completed transitions after 20 observations. It retains at most four candidates
within 0.02 normalized-residual units of the smallest quantile, with deterministic
index tie-breaking. Outside a recovery period, an upward crossing of
$2q_\alpha$ by the maximum candidate residual restores the full candidate set
and suspends pruning for 20 completed transitions. Per-step grade filtering
continues; pruning resumes at the next eligible ten-transition update. A separate
100-transition nominal-residual window updates the operating margin while
preserving $q_\alpha$ as a floor.

RACF and CBF-QP use SLSQP with at most 40 iterations and
$\mathrm{ftol}=10^{-10}$ over $a\in[-1,1]$. A post-solve constraint tolerance
of $10^{-8}$ determines feasibility; infeasible or unsuccessful solves apply
$a=-1$. Per-trial outputs retain solver status, iterations, fallback, action
correction, constraint margin, latency, and the complete experiment configuration.

\clearpage
\subsection{Predictor-alignment study}
This alignment experiment uses different constraints from the equal-cardinality
comparison in Sec.~\ref{sec:reviewer-causal}; the two cohorts are analyzed separately.
These ACC \emph{partial rollouts} use three seed-0
policies (IQL, CQL, BC), two lead scenarios, three physical conditions
(nominal, mass shift, and actuator gain $0.7$), ten development and ten
evaluation resets per cell, and $H=500$. The experiment contains 2,220
method--episode rows and 1,080,508 observed steps; all rows and early
terminations are retained in the accompanying manifest.

The aligned variant explicitly keeps the nominal predictor $f_0$ in the
constraint set and uses $m=L_hq_\alpha=0.16563$. It reaches 149/180
(82.78\%) episode-safe evaluation episodes. Static RACF reaches 144/180
(80.00\%), while Robust CBF-QP reaches 170/180 (94.44\%) in this smaller
cohort. The paired aligned--static difference is +2.78 percentage points
(6 favorable and 1 adverse pair); this is an empirical operating-point result,
not a closed-loop guarantee.

\begin{figure*}[t]
\centering
\includegraphics[width=\textwidth]{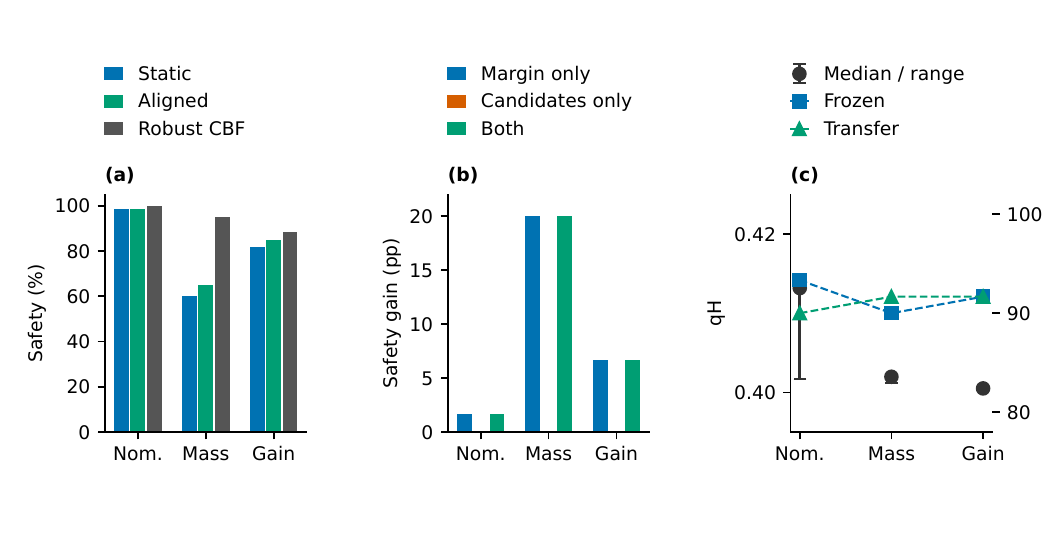}
\caption{Predictor-alignment and component comparisons in ACC \emph{partial rollouts}. (a) Safety
for Static, predictor-aligned RACF, and Robust CBF-QP across nominal, mass, and
gain conditions. (b) Candidate-only, margin-only, and joint updates relative to
Static. (c) Episode-max $q_H$ median and range (left axis) and stopped-episode
coverage in percent (right axis);
the transfer series changes the controller and is descriptive.}
\label{fig:review-alignment}
\end{figure*}

The candidate-only ablation has the same 144/180 safety count as Static,
whereas the margin-only and joint updates each reach 161/180 (89.44\%).
Their paired differences against Static are both +9.44 percentage points;
candidate-only has no discordant pair. The window-20 sensitivity reaches
108/120 (90.00\%) and top-2 selection 109/120 (90.83\%) on the nominal/mass
subset. These values support a component-level sensitivity statement within
this cohort, not a general adaptive-validity claim.

\begin{table}[!t]
\caption{Predictor-alignment methods on 180 evaluation episodes. Safe is the
headway-plus-collision endpoint; FB is fallback steps divided by observed
steps. $J_{95}$ is the mean per-episode jerk P95.}
\label{tab:review-alignment}
\centering
\setlength{\tabcolsep}{1.4pt}
\ninept
% Generated from validated Job 81161; partial rollout.
\begin{tabular}{lrrrrr}
\toprule
Method & Safe/$N$ & Coll. & $|\Delta a|$ & FB (\%) & $J_{95}$ \\
\midrule
Raw & 28/180 & 82 & 0.0000 & 0.00 & 0.76 \\
Base projection & 124/180 & 0 & 0.0152 & 1.32 & 2.13 \\
Static RACF & 144/180 & 0 & 0.0151 & 1.32 & 2.20 \\
Aligned RACF & 149/180 & 0 & 0.0155 & 1.36 & 2.22 \\
Robust CBF & 170/180 & 0 & 0.0146 & 0.89 & 2.03 \\
Envelope RACF & 164/180 & 0 & 0.0158 & 1.36 & 2.28 \\
Episode transfer & 170/180 & 0 & 0.0165 & 1.41 & 2.23 \\
Margin only & 161/180 & 0 & 0.0158 & 1.37 & 2.22 \\
Candidates only & 144/180 & 0 & 0.0153 & 1.34 & 2.26 \\
Both components & 161/180 & 0 & 0.0158 & 1.38 & 2.26 \\
\bottomrule
\end{tabular}

\end{table}

\begin{table}[!t]
\caption{Episode-max calibration range by physical condition. Values pool the
six policy--scenario strata per condition; coverage counts are descriptive.}
\label{tab:review-qh}
\centering
\setlength{\tabcolsep}{3pt}
\ninept
\begin{tabular}{lccc}
\toprule
Condition & $q_H$ range & Frozen & Transfer \\
\midrule
Nominal & 0.4016--0.4133 & 56/60 & 54/60 \\
Mass +20\% & 0.4012--0.4020 & 54/60 & 55/60 \\
Gain $0.7$ & 0.4005--0.4005 & 55/60 & 55/60 \\
\bottomrule
\end{tabular}
\end{table}

The diagnostic calibrates stopped-episode maxima within 18 pre-frozen
policy--condition--scenario strata, including collision-terminated failures.
These follow the declared stopping rule and are not complete-$H=500$ maxima
for terminated episodes. The resulting $q_H$ values range from 0.4005 to 0.4133.
These values exceed the one-step
$q_\alpha$ and hence do not meet $q_H\leq q_\alpha$, the condition obtained
from the margin-fit requirement $m_t\geq L_hq_H$ when $m_t=L_hq_\alpha$.
The displayed rates therefore serve as operating-point evidence rather than a
pathwise guarantee.
The frozen aligned controller's episode-max coverage is reported alongside the
transfer controller in Fig.~\ref{fig:review-alignment}; changing the margin
changes the trajectory law, so the latter remains an empirical transfer test.

\section{Additional experiments and analyses}
\label{sec:evidence-addendum}

This section examines component effects, coverage, dynamics shifts and fallback
behavior. Simulator results are ACC \emph{partial rollouts}; analyses of
existing data and mock measurements are identified explicitly. Experiment
identifiers link the reported counts to the accompanying records.

\subsection{Unified endpoint and fallback definitions}
The primary episode-safe endpoint used by Table~2 and reset-matched recovery
is satisfied when every executed transition has nonnegative headway (gap)
margin and the episode has no collision. A collision is recorded separately
even when an episode terminates early. The dual endpoint, which additionally
requires the speed margin, is retained as a labelled diagnostic in fixed-scope
and replay tables; it is not silently substituted for the Table~2 endpoint.
Projection is the fraction of executed steps with the solver's
\texttt{projected} flag (or the equivalent applied--policy action difference
in the A100 baseline trace); action correction is the mean
$|a_t-a_t^\pi|$; jerk is the per-episode 95th percentile of absolute jerk;
fallback is the fraction of executed steps for which the solver/post-check
path invokes the emergency action. These definitions are applied to the
2$\times$2 ablation, matched comparisons, ACI, and CPSF-style experiments. The initial transition records
did not type fallback causes. Job 82163 therefore reruns the frozen evaluation
units with post-solve telemetry only and reports the scoped exact taxonomy
below; the original actions and trajectories remain bit-identical.

\subsection{Candidate--margin 2$\times$2 ablation}
The four methods reuse 180 policy--condition--scenario--trial units from Job
81161, with identical resets, seeds, solver, bounds, and fallback. The binary
endpoint and continuous paired summaries are computed from per-trial records.

\begin{table*}[t]
\centering\ninept
\caption{2$\times$2 ablation on 180 shared evaluation units. Continuous
values are treatment minus reference; intervals are shared-cluster bootstrap
95\% intervals.}
\label{tab:addendum-2x2}
\begin{tabular}{lrrrrr}
\toprule
Contrast & Safe$_\mathrm{ref}$ & Safe$_\mathrm{trt}$ & $\Delta$ safe (pp) & Proj. $\Delta$ & FB $\Delta$ \\
\midrule
Candidate only $-$ static & 144 & 144 & 0.00 [0,0] & $-0.00049$ & $+0.00017$ \\
Margin only $-$ static & 144 & 161 & +9.44 [4.44,15.56] & $-0.00107$ & $+0.00048$ \\
Both $-$ margin only & 161 & 161 & 0.00 [0,0] & $+0.00026$ & $+0.00007$ \\
Both $-$ candidate only & 144 & 161 & +9.44 [4.44,15.56] & $-0.00032$ & $+0.00038$ \\
\bottomrule
\end{tabular}
\end{table*}

The corresponding mean action-correction differences are $+0.00019$,
$+0.00075$, $-0.00008$, and $+0.00047$; mean jerk-P95 differences are
$+0.0616$, $+0.0206$, $+0.0415$, and $+0.0004$ in the same order. Thus the
observed binary gain in this cohort is attributable to margin adaptation; the
candidate update has no separate observed binary gain.

For the predictor-alignment comparison at $q_\alpha$, the shared-reset
cluster-bootstrap interval is [18.89, 28.89] pp. Resampling individual rows
instead gives [17.78, 30.56] pp; we use the cluster interval to preserve
dependence among policies evaluated on the same resets.

\subsection{Intervention-matched Pareto comparison}
Job 82029 evaluated 1,920 rows (three policies, two physical conditions, two
lead scenarios, eight operating points, ten development and ten independent
test trials). Static margins were selected on development projection rates.
The predefined matching criteria (projection-rate error at most 0.005 and relative action
correction error at most 0.10) retained 40 of 48 adaptive--static choices,
yielding 400 paired test units; eight unmatched choices remain in the records
but are excluded from this comparison. Figure~\ref{fig:addendum-pareto}
shows three safety--utility views: safety--intervention, safety--jerk, and
safety--correction views. This is a partial rollout and does not establish a
global Pareto frontier.

\begin{figure*}[t]
\centering
\includegraphics[width=\textwidth]{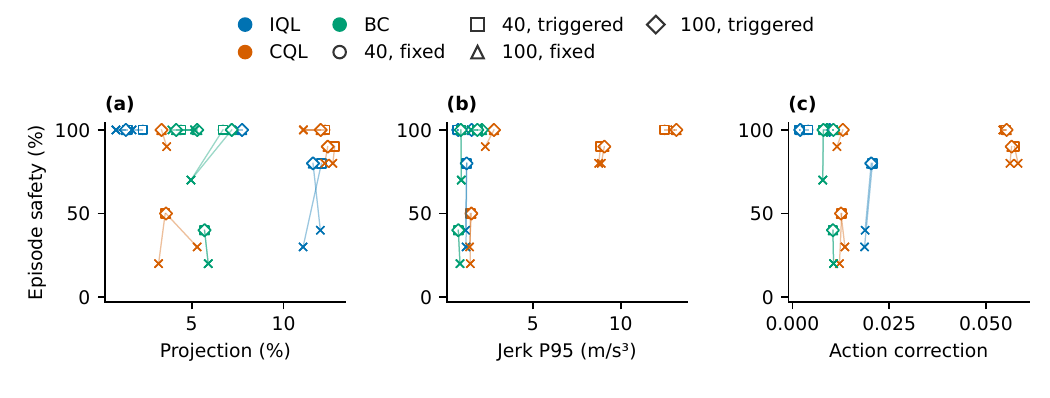}
\caption{Development-matched Pareto comparisons on independent test seeds.
(a) Safety versus projection frequency. (b) Safety versus episode $P_{95}$ jerk.
(c) Safety versus mean action correction. Crosses are Static operating points,
open markers are Adaptive points, and lines connect each eligible pair. Colors
identify policy families. Marker labels give the residual-window length (40 or
100) and fixed-threshold or residual-triggered updating.}
\label{fig:addendum-pareto}
\end{figure*}

\subsection{Predictor alignment and fallback analysis}
The predictor-aligned RACF analysis uses 90,000 recorded evaluation steps. The nominal
$f_0$ is explicitly present at all steps, $m=L_hq_\alpha=0.1656329484$ at all
steps, and no strict implication failure is observed. Safety is 149/180;
fallback occurs on 1,224/90,000 steps (1.360\%). The episode-calibrated
transfer variant also includes $f_0$ at all steps, but changes the closed-loop
policy through $q_H$-dependent margins; it is an empirical transfer diagnostic
(170/180 safe, 1,265/90,000 fallback steps), not a guarantee. Emergency
fallback can leave a negative robust constraint. The original trace does not
itself support a typed decomposition; the separately instrumented Job 82163
below supplies that decomposition without changing the action path. These
limitations restrict interpretation of the observed safety results.

\subsection{Rolling coverage and adaptive conformal baseline}
The isolated ACI baseline uses the first 500 nominal calibration scores,
$\alpha_0=0.10$, $\gamma=0.01$, and the same state access, candidate grid,
solver, and fallback. On 180 paired evaluation episodes it is 143/180 safe,
versus 144/180 for static RACF (difference $-0.56$ pp, cluster-bootstrap CI
$[-1.67,0.00]$) and 161/180 for adaptive RACF (difference $-10.00$ pp, CI
$[-14.44,-6.11]$). Descriptive rolling coverage is 0.8997 before step 100
and 0.8675 after step 100 with a 50-step window; the episode miscoverage
mean is 0.1261. Figure~\ref{fig:addendum-coverage} reports the stepwise
trajectory. These values do not establish adaptive conformal validity or a
closed-loop safety guarantee.

\begin{figure}[!h]
\centering
\includegraphics[width=\columnwidth]{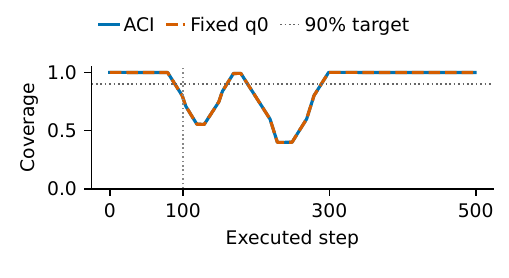}
\caption{Descriptive rolling coverage from the ACI partial rollout. The
blue ACI and orange fixed-$q_0$ curves use a 50-step rolling window and coincide
where the coverage values agree. The vertical line marks step 100 and the
dashed horizontal line the 0.90 target.}
\label{fig:addendum-coverage}
\end{figure}

\subsection{External CPSF-style comparator}
The external baseline is a one-step finite predictive-set comparator:
$\mathcal C_t=\{\theta_i:\|f_{\theta_i}(s_t,a_t^\pi)-f_0(s_t,a_t^\pi)\|\le q_0\}$,
followed by the unchanged RACF projection and fallback. It is not a
multi-step backup/viability predictive safety filter reproduction. Job 82045
completed 180 paired episodes and 90,000 steps. It obtains 144/180 safety,
equal to static RACF (paired difference 0.00 pp, CI [0,0]) and below adaptive
RACF by 9.44 pp (CI $[-13.89,-5.56]$). Fallback is 1,192/90,000 steps
(1.324\%); no predictive set was empty in this cohort. The result is an
empirical external-baseline comparison only.

\subsection{Dynamic recovery and out-of-grid evidence}
The dynamic trace reconstruction uses 180 existing adaptive traces and 360
observable strong-residual proxy events. Margin response occurs for 360/360
events, candidate broadening for 249/360, first action intervention for
155/360, and ten-step residual recovery for 360/360. These are observable
proxy alignments, not causal selector-trigger labels. Figure~\ref{fig:addendum-dynamic}
shows one selected time window.

\begin{figure}[!h]
\centering
\includegraphics[width=\columnwidth]{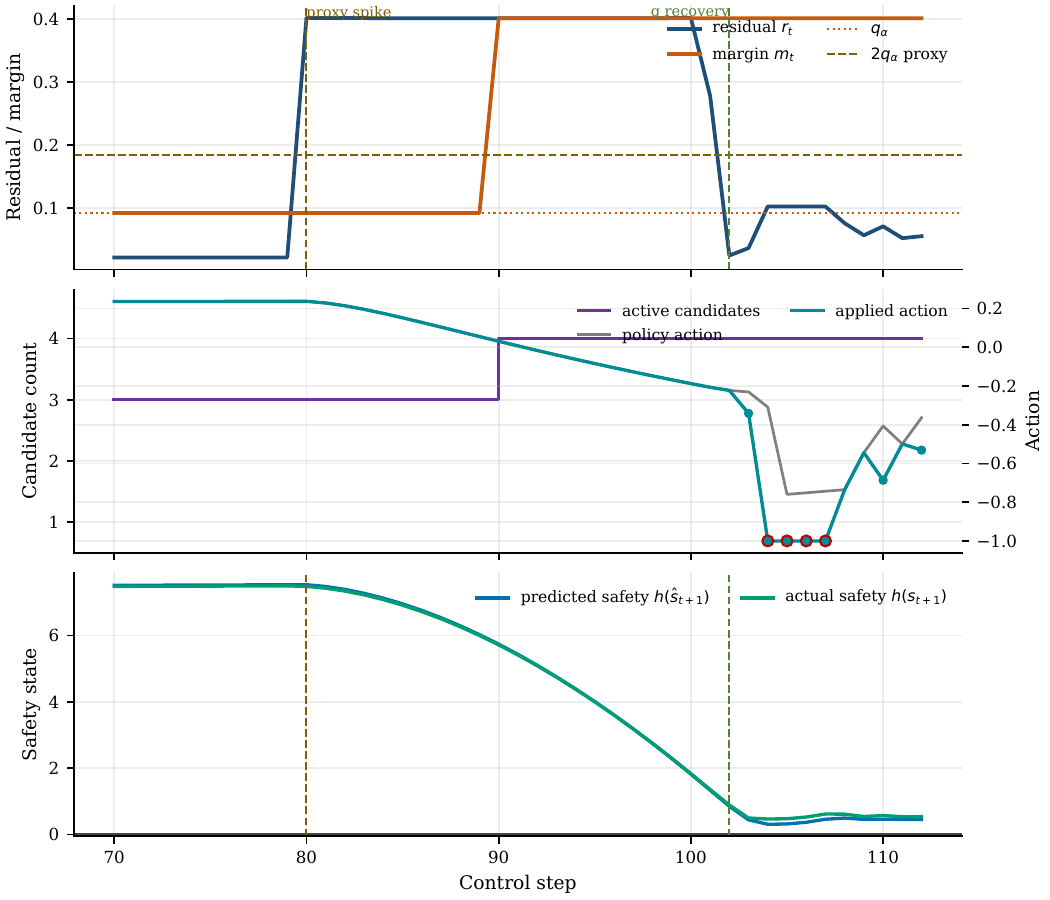}
\caption{Dynamic-response trace reconstructed from existing
serialized observables. (a) Residual $r_t$, margin $m_t$, $q_\alpha$, and the
$2q_\alpha$ proxy threshold. (b) Active-candidate count and policy/applied
actions; red-outlined points mark emergency braking. (c) Predicted and actual
safety margins. Vertical lines mark the proxy spike and residual recovery.
The sequence is descriptive and does not imply that fallback preserves safety.}
\label{fig:addendum-dynamic}
\end{figure}

The completed mass development cohort (Job 81979) contains 1,000 rows over
1,200, 1,500, 1,800, and 2,100 kg; 1,200 and 2,100 kg are extrapolation
points relative to the registered mass range. It has 46 collision rows and is
reported only as development partial-rollout evidence. The four-axis
extension (Job 82049) covers mass, grade, actuator gain, and lead braking in
3,750 method--episode observations under its separate frozen protocol. Collision-free and
controllable-safe endpoints are reported separately; this remains development
partial-rollout evidence rather than a generalization or full-benchmark claim.

\begin{figure}[!h]
\centering
\includegraphics[width=\columnwidth]{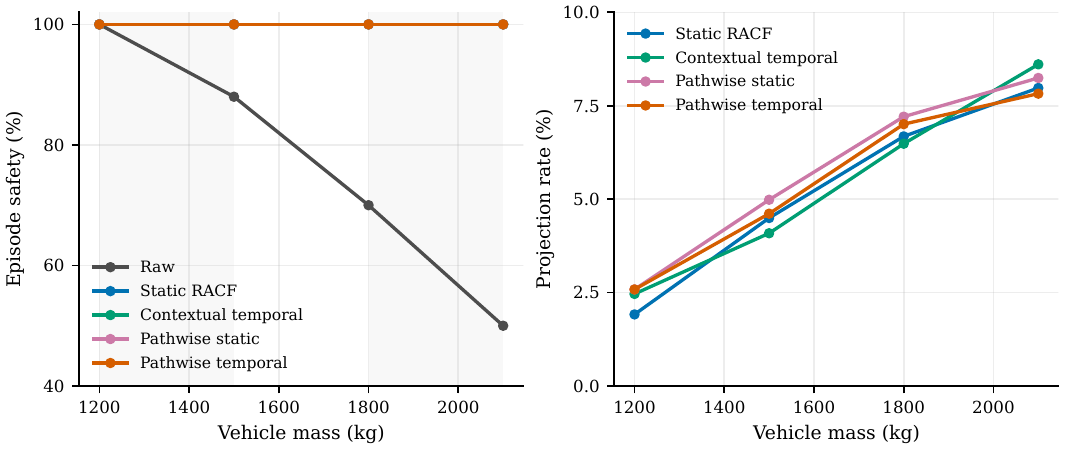}
\caption{Mass-axis development diagnostic. (a) Episode-safe rate and
(b) projection rate for Raw and four RACF operating points. Masses 1500 and
1800 kg are registered; 1200 and 2100 kg are extrapolation points. This
development partial rollout is not a generalization guarantee.}
\label{fig:addendum-mass}
\end{figure}

\begin{figure}[!h]
\centering
\includegraphics[width=\columnwidth]{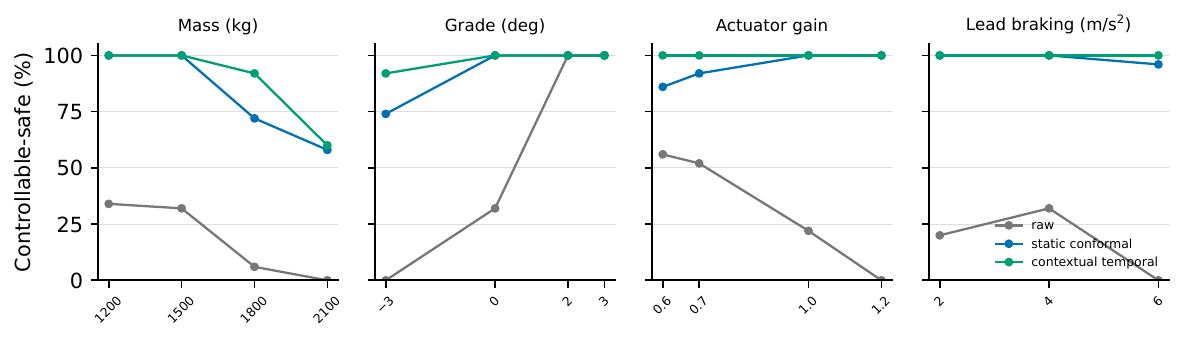}
\caption{Four-axis out-of-grid development summary from Job 82049. Panels show
the controllable-safe rate for Raw, Static conformal, and Contextual temporal
across mass, grade, actuator gain, and lead braking. Collision-free status,
projection, fallback, correction, jerk, and paired cluster-bootstrap contrasts
are reported in the numerical results rather than plotted here.}
\label{fig:addendum-four-axis}
\end{figure}

\subsection{Runtime and utility boundary}
The real CPU episode probe covers 54 ACC episodes (three policies, three
conditions, two scenarios, three trials) and 27,000 executed steps. Wall-clock
episode latency has mean/P50/P95/P99/max 1,030.8/1,039.5/1,424.4/1,590.4/
1,619.6 ms. The warm per-step policy--filter end-to-end stage has an episode
mean of 1.992 ms and an episode-mean P95 of 2.781 ms; its mean filter-only and
update stages are 1.591 and 0.234 ms. This is a real CPU partial-rollout
measurement, but it remains host-specific and is not a target-hardware or
full-benchmark latency claim. The frozen-state CPU mock probe is retained in
the evidence package for comparison. Existing utility tables report safety,
collision, violation burden, tracking error, jerk, action correction, and
projection frequency; fallback safety is kept separate because the analysis shows
that emergency fallback can violate the robust constraint.

\subsection{One-factor sensitivity sweep}
Job 82109 evaluates 1,040 episodes from one frozen IQL checkpoint: four
physical conditions, two lead scenarios, five development and five independent
test trials, and 13 configurations. The default is frozen at window 40,
update interval 10, $\eta=1$, top-$k=4$, pruning slack 0.02, grade threshold
$0.75^\circ$, restore threshold $2q$, and recovery length 20. Alternatives
vary one factor at a time: window 100; update intervals 5 and 20; $\eta=0.5$
and 1.5; top-$k=2$; pruning slack 0.01 and 0.05; grade thresholds 0 and
$1.5^\circ$; and restore thresholds $q$ and $3q$. Development ranking uses
safety, then violation rate, action correction, and jerk; independent test rows
are never used to select a configuration. The test endpoint is unchanged across
these settings in this cohort, while projection, fallback, correction, and jerk
differences are retained in the machine-readable analysis.

\begin{figure*}[t]
\centering
\includegraphics[width=\textwidth]{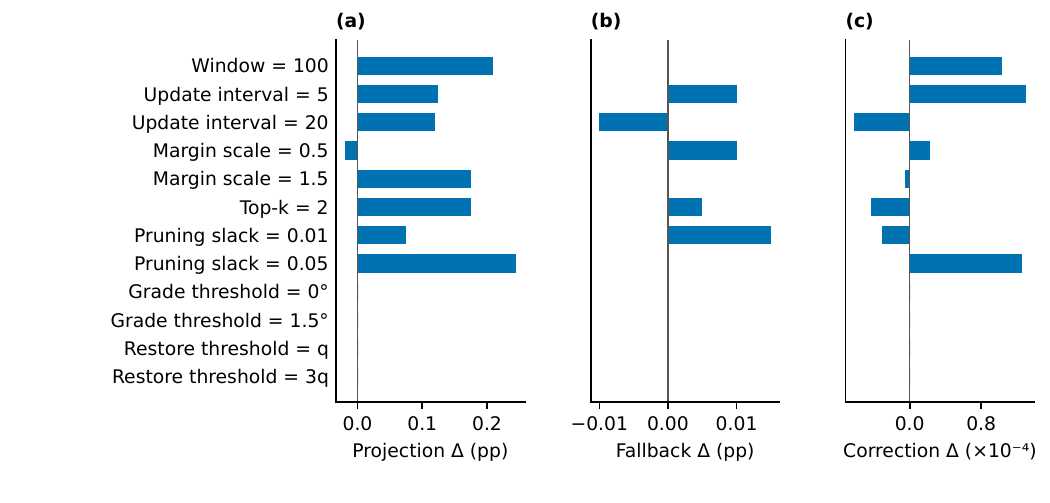}
\caption{Independent-test deltas relative to the frozen sensitivity baseline.
Rows vary one hyperparameter at a time; bars show setting-minus-baseline
changes in projection and fallback (pp) and mean action correction
($10^{-4}$ normalized action units). Safety
endpoint deltas and cluster-bootstrap intervals are retained in the
accompanying JSON.}
\label{fig:addendum-sensitivity}
\end{figure*}

\subsection{Fallback causes and outcomes}
Job 82163 is a 540-episode, 270,000-step ACC partial rollout comprising 180
development traces and the same 180 evaluation units for predictor-aligned and
adaptive RACF. It freezes policies, checkpoints, reset seeds, conditions,
scenarios, horizon, SLSQP options, action bounds, and emergency action. The
instrumentation runs after the unchanged solver/action path; all 540 traces are
state- and applied-action-identical to Job 81161 (maximum differences zero).

For the registered scalar ACC action $a\in[-1,1]$, every active one-step
constraint is non-increasing in $a$. Hence the continuous feasible intersection
is empty exactly when the complete constraint vector at $a=-1$ fails the
$-10^{-8}$ tolerance. Under this scoped predicate, predictor-aligned RACF has
1,224 fallback invocations in 90,000 evaluation steps: 812 feasible-set-empty,
412 solver-failure, and zero post-check-failure. Adaptive RACF has 1,241:
808, 433, and zero, respectively. The corresponding numbers of episodes with
fallback are 140/180 and 145/180. No fallback step is a collision; 439 aligned
and 276 adaptive feasible-set-empty fallback steps have a simultaneous gap
violation, while solver-failure steps have none. This exactness statement is
limited to the registered scalar one-step ACC model and tolerance; it is not a
generic RACF or closed-loop safety guarantee. The earlier 401-action analysis
is retained as an analysis-only historical cross-check and gives identical
counts in this cohort.

\section{Additional Evidence and Reproducibility}
Main Fig.~2 summarizes five complementary comparisons. This supplement provides
the complete method table, horizon, condition/scenario and family slices,
protocols, traces, and continuous utility details.
Figure~\ref{fig:registered-slices} separates horizon, physical-condition, lead-
scenario, and policy-family evidence. Table~\ref{tab:interface-distinction}
states the intended scope of the three uncertainty-aware interfaces.

\subsection{Evaluation cohorts and data sources}
The residual-margin isolation cohort and the complete registered comparison are
distinct experiments. The former is Job~80915 (`racf\_pre\_racf\_baseline\_v1`)
with seed base 380000000, 12 policies, four conditions, two scenarios, 25
resets per cell, and 4,800 method rows; its Base/Static counts are 1993/2400
and 2102/2400. The latter reuses the fixed Job~79931 evaluation set and adds the
Job~82381 extension under `racf\_full\_registered\_acc\_matrix\_v1`; it has
12 policies, four conditions, two scenarios, 25 resets per cell, nine methods,
and 21,600 rows; its Static row is 2133/2400. These same-name rows do not share
a verified trial-key set and are not pooled. The package includes the two
protocols, manifests, source hashes, and document why the cohorts are not pooled.

\subsection{Matched-window mechanisms and computation}
\label{sec:reviewer-causal}
Job~82722 comprises three experiments in one A100 allocation. The margin study uses the
same three frozen policies, nominal/mass/gain conditions, brake/stop--go
scenarios, solver, eight-candidate set, state access, and fallback. Five
development resets per cell select a fixed margin from
\{0,0.05,$q_\alpha$,$L_hq_\alpha$,0.25,0.35,0.45,0.60\} and clipped-ACI
$\gamma\in\{0.005,0.01,0.02\}$ before 10 disjoint test resets per cell.
Fixed, rolling without/with an offline floor, ACI without/with that floor,
Adaptive, Envelope, and Robust reach 162, 156, 156, 157, 157, 158, 159, and
171 safe episodes out of 180; all have zero collisions. The floor changes no
binary outcomes. Adaptive minus fixed is $-2.22$ pp (95\% shared-cluster CI
$[-5.56,0.56]$), whereas Robust minus Adaptive is $+7.22$ pp
($[2.78,12.22]$). Thus neither a conformal-specific nor an update-rule safety
advantage is identified. Fixed, rolling and ACI controls use the full grid;
only the four rolling/ACI arms use the strict $r_t>q_t^{\rm raw}$ miss and
100-step window updated per step after warm-up;
standard Adaptive retains candidate selection and ten-step updates. The floor affects only control; ACI feedback uses
$q_t^{\rm raw}$ and $\alpha_t$ updates during warm-up. All 18,000 paired
warm-up steps are identical. Floor activation of 45.7\%/46.2\% raises coverage
from 83.6/84.3\% to 90.4/90.8\% but yields zero discordant safety pairs, an
observed null rather than population equivalence.

The runtime study uses one A100-PCIE-40GB GPU, one inference process and eight allocated
CPUs; OMP, MKL, OpenBLAS and NumExpr thread limits are one. The CPU model and
PyTorch inter-op thread count were not logged. Balanced method order shares the
same software and host allocation with two other experiments.
Wall time spans episode setup/reset, policy inference, filtering, updates,
simulator steps, in-memory diagnostics and outcome aggregation; checkpoint
loading and trace-file compression are excluded. Per-episode wall time divided
by executed steps is averaged across episodes. Adaptive, Envelope, and Robust average
2.434, 4.082, and 3.090 ms per executed step and reach 161/180, 164/180, and
170/180 safety. Adaptive retains the closest residual candidate on 98.11\% of
steps with mean residual regret $1.54\times10^{-5}$. The alignment study holds
cardinality at one: exact $f_0$ and a
parameter-matched callable are endpoint-identical on all 180 pairs (133/180).
A safety-direction residual gives $q_h=0.08193$; margins $q_\alpha$,
$L_hq_\alpha$, and $q_h$ yield 138/180, 143/180, and 139/180 safety. Hence the
directional variant does not replace the present method. The accompanying
records contain protocols, scripts, hardware specifications, input manifests,
paired intervals and SHA256 checksums. All simulator results are
\emph{partial rollouts}, not a closed-loop certificate or external benchmark.

\begin{figure*}[t]
\centering
\includegraphics[width=\textwidth]{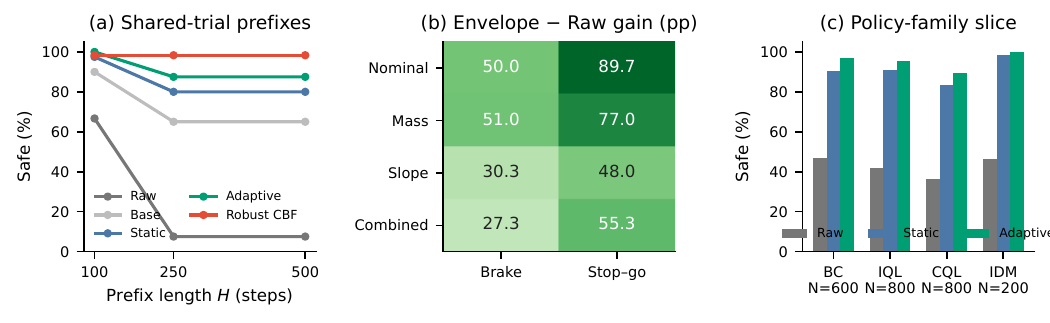}
\caption{Registered supporting slices. (a) Safety on shared trial prefixes;
all horizons use fixed denominators and collision-stopped trajectories remain
unsafe at longer prefixes. (b) Envelope--Raw safety gain over four physical
conditions and two lead scenarios. (c) Raw/Static/Adaptive policy-family safety;
denominators appear below each family and exact safe counts are in
Table~\ref{tab:policy-family}.}
\label{fig:registered-slices}
\end{figure*}

\begin{table}[t]
\caption{Operational distinction between uncertainty-aware safety interfaces.
Resid. identifies the calibrated object; Adapt. denotes online updating; Models
denotes explicit predictive-model constraints; Fallback records whether a
bounded emergency action is part of the interface.}
\label{tab:interface-distinction}
\centering
\small
\setlength{\tabcolsep}{3pt}
\begin{tabular}{lcccc}
\toprule
Method & Resid. & Adapt. & Models & Fallback \\
\midrule
Pred. filter & set & -- & model & -- \\
Adapt. conf. & quantile & yes & -- & -- \\
RACF (ours) & margin & yes & limited & bounded \\
\bottomrule
\end{tabular}
\end{table}

\begin{figure*}[t]
\centering
\includegraphics[width=\textwidth]{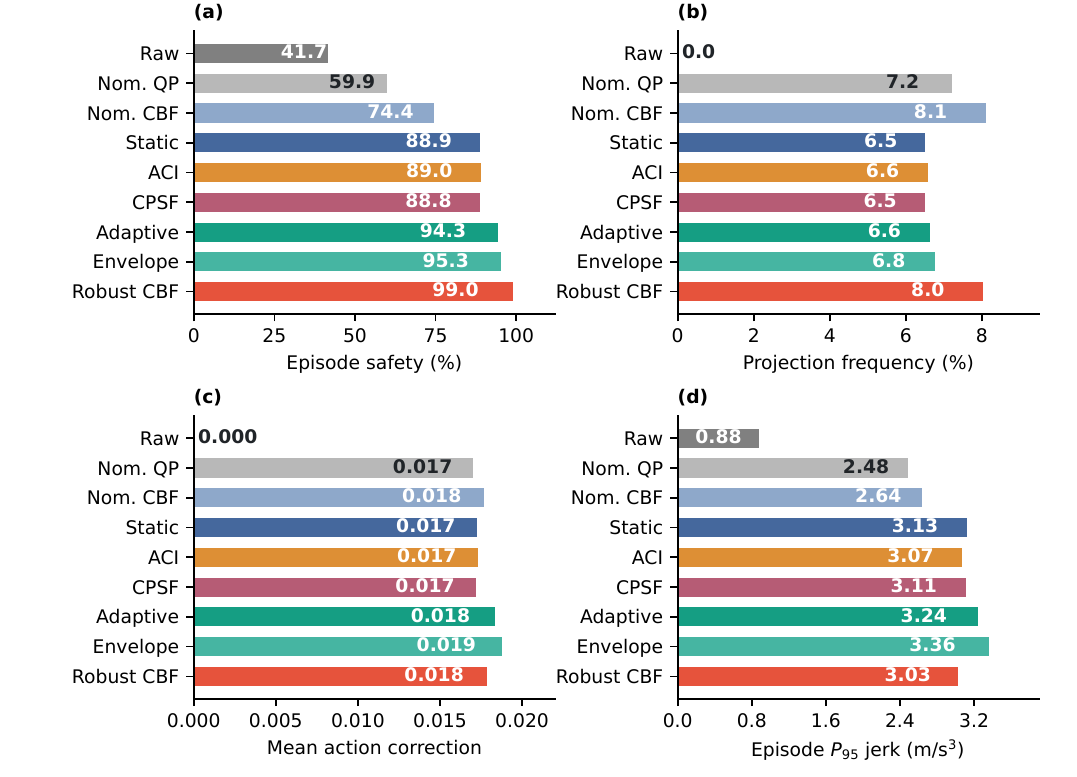}
\caption{Complete nine-method comparison on the shared 2,400-unit ACC
cohort. (a) Episode-safe rate. (b) Fraction of observed steps projected.
(c) Mean absolute normalized action correction. (d) Mean episode $P_{95}$ jerk
(m/s$^3$). ACI and CPSF-style share RACF's policies, resets, trajectories,
solver, and bounded fallback.}
\label{fig:full-registered-utility}
\end{figure*}

\begin{table*}[t]
\caption{Safety--utility results on the shared ACC cohort. Safe is the exact
safe-episode count and SR its percentage; Coll. counts collision-terminated
episodes. Projection and fallback are observed-step percentages, Correction is
the mean absolute normalized action change, and jerk is episode $P_{95}$ in
m/s$^3$. Utility uses observed prefixes, so Raw collisions confound cross-method
utility comparisons through survival/truncation. All filtered methods record
zero collisions. Half-up rounding gives Robust fallback 0.73\% from exactly
8,700/1,200,000 steps (fraction 0.00725).}
\label{tab:full-registered}
\centering
\small
\setlength{\tabcolsep}{4pt}
\begin{tabular}{lrrrrrrr}
\toprule
Method & Safe/2400 & SR (\%) & Coll. & Proj. (\%) & Fallback (\%) & Correction & $P_{95}$ jerk \\
\midrule
Raw & 1001 & 41.7 & 904 & 0.00 & 0.00 & 0.0000 & 0.88 \\
Nominal QP & 1437 & 59.9 & 0 & 7.21 & 0.75 & 0.0170 & 2.48 \\
Nominal CBF-QP & 1785 & 74.4 & 0 & 8.11 & 0.60 & 0.0177 & 2.64 \\
Static RACF & 2133 & 88.9 & 0 & 6.51 & 0.91 & 0.0172 & 3.13 \\
ACI & 2137 & 89.0 & 0 & 6.57 & 0.90 & 0.0173 & 3.07 \\
CPSF-style & 2130 & 88.8 & 0 & 6.49 & 0.91 & 0.0172 & 3.11 \\
Adaptive RACF & 2262 & 94.3 & 0 & 6.63 & 0.95 & 0.0183 & 3.24 \\
Envelope RACF & 2287 & 95.3 & 0 & 6.77 & 0.97 & 0.0188 & 3.36 \\
Robust CBF-QP & 2375 & 99.0 & 0 & 8.02 & 0.73 & 0.0179 & 3.03 \\
\bottomrule
\end{tabular}
\end{table*}

On the shared evaluation cohort, Adaptive RACF exceeds ACI by 5.21 pp (95\% paired
cluster-bootstrap interval [3.96, 6.54]) with 127 favorable and two adverse
pairs. It exceeds CPSF-style by 5.50 pp ([4.25, 6.83]) with 134 favorable and
two adverse pairs. The comparison isolates the residual-to-action interface:
the policies, reset states, exogenous lead trajectories, solver, and fallback are unchanged.
For Adaptive minus Envelope, 95\% intervals for safety, projection, correction,
and jerk are $[-1.58,-0.54]$ pp, $[-0.26,-0.02]$ pp,
$[-0.00053,-0.00035]$, and $[-0.182,-0.046]$. Against Robust, they are
$[-5.88,-3.58]$ pp, $[-1.59,-1.20]$ pp, $[0.00031,0.00063]$, and
$[0.126,0.297]$.

\begin{table}[t]
\caption{Policy-family decomposition of the registered comparison. Policies and
Trials give the family denominator; Raw, Static, and Adaptive report exact
safe-episode counts.}
\label{tab:policy-family}
\centering
\small
\setlength{\tabcolsep}{3pt}
\begin{tabular}{lrrrrr}
\toprule
Family & Policies & Trials & Raw & Static & Adaptive \\
\midrule
BC & 3 & 600 & 282 & 543 & 581 \\
IQL & 4 & 800 & 336 & 726 & 764 \\
CQL & 4 & 800 & 290 & 667 & 717 \\
IDM & 1 & 200 & 93 & 197 & 200 \\
\bottomrule
\end{tabular}
\end{table}

For Fig.~2(d), each of the 48 development choices is one of four adaptive
configurations in a policy--condition--scenario stratum (3 policies, 2
conditions, 2 scenarios). Eligibility requires absolute projection-rate error
$\leq0.5$ pp and relative action-correction error $\leq10\%$ against the
nearest Static point. Forty choices pass; the eight failures are omitted from
the matched-pair plot. Each retained choice is evaluated on 10 disjoint test
resets, giving 400 paired test units.

\begin{table}[t]
\caption{Constraint comparisons on 180 shared units. Safe/total gives exact
paired counts; $\Delta$ pp is the first configuration minus the second. Adding $f_0$
changes both exact-$f_0$ inclusion and constraint cardinality.}
\label{tab:fixed-scope}
\centering
\small
\setlength{\tabcolsep}{3pt}
\begin{tabular}{lcr}
\toprule
Comparison & Safe/total & $\Delta$ pp \\
\midrule
Global fixed / pooled & 170/180 / 170/180 & 0.00 \\
Added $f_0$, $q_\alpha$ / singleton & 140/180 / 97/180 & +23.89 \\
Added $f_0$, $L_hq_\alpha$ / singleton & 143/180 / 126/180 & +9.44 \\
\bottomrule
\end{tabular}
\end{table}

In Table~\ref{tab:fixed-scope}, the added-$f_0$ configuration minus the singleton has 95\% shared-reset
intervals [18.89,28.89] pp at $q_\alpha$ and [5.56,13.89] pp at $L_hq_\alpha$.
Because adding $f_0$ also changes constraint cardinality, these quantify
the paired configuration differences, not isolated alignment effects.

\subsection{Limitations and negative results}
Candidate-only matches Static at 144/180 safe episodes in the fixed cohort,
while margin-only and Joint reach 161/180. ACI and CPSF-style comparisons are
one-step interface comparisons rather than multi-step viability-filter
reproductions. Adaptive and Envelope relax the formal alignment, scaling,
feasibility, or fallback premises, so their safety remains empirical.

\subsection{Assessment of theoretical assumptions}
Reference-score exchangeability and the deployment TV radius $\epsilon$ are
theorem inputs; neither is estimated by the rollout matrix. In the aligned
180-episode study, the exact $f_0$ constraint and $m=L_hq_\alpha$ relation are
trace-checked on all 90,000 steps with no missing or margin-mismatch rows. This
verifies only the structural premises. The same analysis records 1,224 fallback
invocations, including 812 steps where emergency braking leaves a negative
robust constraint, so premise-preserving fallback is not established. The
aligned result therefore remains empirical. An analysis of existing traces checks
$V_t\leq C_t+P_t$ on 90,000 archived steps per method; zero violations occur
with both residual coverage and the nominal premise satisfied. Full $(V,C,P)$,
fallback-step and fallback-episode counts are available in the accompanying
transition-level analysis, which does not estimate
IID/TV assumptions or a closed-loop guarantee.

The added-$f_0$ diagnostic is not a default-grid insertion experiment. Its non-aligned
configuration is a singleton with mass 1500 kg, grade $0^\circ$, and actuator
gain 1.0. The aligned configuration adds the exact score-defining $f_0$.
The development-set diagnostic contains actions for which the
singleton candidate passes and $f_0$ fails at the same margin, confirming
implementation-level non-equivalence. Because cardinality also changes from
one to two constraints, the safety difference is not an isolated alignment
effect.

\subsection{Adaptive execution order}
At step $t$, grade compatibility is applied to the selector's current set before
projection. After each completed transition, margin updates precede selector
updates; both affect step $t+1$. Selection
starts after 20 transitions, uses a 40-transition window, and is recomputed every
ten transitions. Outside a recovery period, an upward crossing of $2q_\alpha$ restores the selector's
full grid for 20 completed transitions; pruning is suspended during this counter,
while the per-step grade filter remains active. Pruning resumes at the next
eligible ten-transition update. ACI changes only its quantile margin;
CPSF-style forms its one-step set at the current policy proposal.

\subsection{Reproducibility and baseline contracts}
The matrix uses four physical conditions, brake and stop--go motion,
25 resets per cell and $H=500$ (Sec.~\ref{sec:supp-protocol}). Policies observe
normalized speed error, closing speed, gap, grade and previous action; filters
receive simulator state. The 500 disjoint nominal calibration transitions give
$q_\alpha=0.0918766$ at $\alpha=0.1$. Episode safety requires nonnegative
headway margin and no collision ($g\leq0.5$ m). Infeasibility or failed checks
invoke bounded braking.

ACI initializes $\alpha_0=0.10$, forms $I_t=1\{e_t>q_{\alpha_t}\}$, updates
$\alpha_{t+1}=\Pi_{[.01,.50]}[\alpha_t+0.01(0.10-I_t)]$, and recomputes the
quantile from the fixed 500-score pool while retaining all eight candidates.
CPSF-style uses
$\mathcal C_t=\{\theta:\|f_\theta(s_t,a_t^\pi)-f_0(s_t,a_t^\pi)\|_2\leq q_0\}$
with $q_0=0.0918766$; an empty set restores the full grid. Both interfaces
share models, solver and fallback; neither reproduces a multi-step filter.